\documentclass[preprint,12pt]{elsarticle}

\usepackage{amssymb}
\usepackage{amsmath}

\journal{ToC}

\newtheorem{Remark}{\sc Remark}[section] 
\newtheorem{Theorem}{\sc Theorem}[section]
\newtheorem{Corollary}[Theorem]{\sc Corollary}

\newtheorem{Definition}{\sc Definition}[section]
\newtheorem{Hypothesis}[Theorem]{\sc Hypothesis}

\begin{document}

\begin{frontmatter}



\title{Turing or Cantor: That is the Question} 


\author{Eugene Eberbach\corref{cor1}\fnref{label2}} 
\ead{eeberbach@gmail.com}
\fntext[label2]{Corresponding Author}
\cortext[cor1]{Retired Professor of Practice (Rensselaer Polytechnic Institute)}
\affiliation{organization={\;\;\;\;\;\;\;\; Dept. of Eng. and Science, Rensselaer Polytechnic Institute}, 
            city={Hartford},
            state={CT},
            country={USA}}

\begin{abstract}
Alan Turing is considered as a founder of current computer science together with Kurt G\"{o}del, Alonzo Church and John von Neumann. In this paper multiple new research results are presented. It is demonstrated that there would not be Alan Turing's achievements without earlier seminal contributions by Georg Cantor in the set theory and foundations of mathematics. It is proposed to introduce the measure of undecidability of problems unsolvable by Turing machines based on probability distribution of its input data, i.e., to provide the degree of unsolvabilty based on the number of undecidable instances of input data versus decidable ones. It is proposed as well to extend the Turing's work on infinite logics and Oracle machines to a whole class of super-Turing models of computation. Next, the three new complexity classes for TM undecidable problems have been defined: U-complete (Universal complete), D-complete (Diagonalization complete) and H-complete (Hypercomputation complete) classes. The above has never been defined explicitly before by other scientists, and has been inspired by Cook/Levin NP-complete class for intractable problems. Finally, an equivalent to famous
$P \;?= \; NP$ unanswered question  for $NP$-complete class, has been answered negatively for U-complete class of complexity for undecidable problems.
\end{abstract}




\end{frontmatter}




\section{Introduction}\label{sec1}

Alan Turing was born in 1912, became an undergraduate in mathematics at King's
College, the University of Cambridge, England in 1931, and
published a paper in 1936 whose model of Turing machines became a central
basis for the theory of computing.
This paper, entitled ``On Computable Numbers with an Application to the
Entsheidungsproblem'' \cite{turing36} was primarily a mathematical demonstration that
computers could not model all mathematical problem solving.
In his paper, Turing introduced his discrete state automatic machines ($a$-machines, known currently as Turing machines) used to prove the undecidability of the famous decision Entscheidungsproblem of Hilbert whether mathematics is complete and all its statements (theorems) can be decided to be true or false. 

In this paper it is stressed the genius of Alan Turing in creation of digital programmable computers, formalization of algorithms,  proving that his Turing machines cannot solve by finite means all problems, and mathematics is incomplete. Turing did that by inventing a new automatic machine model of computation (known now as Turing machine TM) and proving that some real numbers cannot be computed by TMs. The above follows directly from earlier work of Georg Cantor who proved that the cardinality of real numbers is strictly larger than the cardinality of integers.  The job of Turing was simply to associate the number of all Turing machines with integers, and the number of all languages/problems with real numbers. That's it - the rest is the history.  The elegance and simplicity of Turing's approach turned out to be equivalent to disproving the solvability of the famous Hilbert's Entscheidungsproblem (decision problem) in mathematics and provided the limit to the algorithmic problem solving by computers. 
As the follow up of the Turing's proof of the undecidability of Entscheidungsproblem  three new complexity classes for TM undecidable problems have been defined: U-complete (Universal complete), D-complete (Diagonalization complete) and H-complete (Hypercomputation complete) classes.

This paper concentrates on the famous Alan Turing's blueprint from 1936 \cite{turing36}
fulfilling in computer science the function of providing the foundations of computers science and computers. For that paper alone, Alan Turing has been recognized as the father of computer science. 
If so, in this paper, it has been justified that perhaps another famous mathematician,  Georg Cantor might be considered an implicit forgotten grandfather of computer science together with Charles Babbage and Wilhelm Leibnitz.

\cite{turing36} is a very interesting and strange paper. It is written almost entirely in plain English, it does not have an abstract, keywords, proper introduction with previous work done, explicit theorems and definitions, practically no references, it contains several errors (corrected later), no conclusions and future work despite that it is considered commonly as the starting point of the era of programmable digital computers. 
In the opinion of the author of this paper, Turing's paper was (and still is) too original and unconventional to be accepted by the majority of current professional journals, because it violated several rigid formal requirements. The question remains open, how ingenious Turing's paper really was that despite of all the above problems, it remains to be so tremendously successful and influential for so many years.

The \emph{Entscheidungsproblem}
(decision problem), one of the most
alluring conjectures in mathematics, proposed by the prominent 
mathematician David Hilbert in 1928 \cite{hilbert28}. Hilbert's conjecture that any mathematical proposition could be 
{\em decided} (proved true
or false) by mechanistic logical methods was approved by many mathematicians, but was unexpectedly disproved by
G\"odel in 1931 \cite{godel31}, who showed that for any formal theory, there will always be
undecidable theorems outside of its reach (famous {\em incompleteness theorem}). 

Mathematicians like Alonzo Church
\cite{church36}
and Turing \cite{turing36} continued G\"odel's work, looking for alternate techniques for
proving this undecidability result.
Turing's disproof of Hilbert's {\em Entscheidungsproblem} was accepted by G\"odel as
better than his own proof, and was also accepted by Church as better than his own
attempt to show that his $\lambda$-calculus could not prove the
{\em Entscheidungsproblem} either \cite{church36}.

Turing's proof, provided in his 1936 paper \cite{turing36} was based on a
novel model of \emph{automatic machines} ($a$-machines), which can be built
to carry out any algorithmic computation. 
Turing showed that despite their versatility, these machines 
cannot compute all real numbers (extended to functions of integers, functions of real/computable variables, computable predicates, etc.), in particular, he proved that 
the now-famous \emph{halting problem} of the Universal Turing Machine (representing all Turing machines) is undecidable.
Turing defined {\em ``computable''} numbers at the beginning of his paper as the real numbers whose expressions as a decimal are calculable by finite means\footnote{Turing justified that by the fact that the human memory is neccesarily limited. However, it does not prevent us to think and operate with infinite objects, and that the contents of our memory fluctuates by forgetting and learning}.

In the 1960s,  with the development of new ACM computer science curriculum based on algorithms,
Turing Machines were chosen to fulfill a central role in computer science, because of its
elegance and simplicity, and for many years they have been useful in that function. 

However, we forget typically that 
Turing was more concerned that the {\em Entscheidungsproblem} is not solvable, and mathematics cannot be fully
described by recursive algorithms, rather than to propose or create the foundations of computer science.
Turing himself did not accept Turing machines as a complete model for problem solving, proposing 3 other models of computation going beyond his $a$-machines, and its
acceptance by theoreticians contradicted Turing's view on this subject.

The justification why we need new models of computation besides TMs are the following. In \cite{denning11}  Peter Denning expressed concerns that developments in non-terminating computation, analog computation, continuous computation, and natural computation may require rethinking the basic definitions of computation. He further stated that computation is the process that the machine or algorithm generates. In analogies: the machine is a car, the desired outcome is the driver's destination, and the computation is the journey taken by the car and driver to the destination. 

On the other hand, Alfred Aho \cite{aho11} wrote ``as the computer systems we wish to build become more complex and as we apply computer science abstractions to new domains, we discover that we do not always have the appropriate models to devise solutions. In these cases, computational thinking becomes a research activity that includes inventing appropriate new models of computation''. It looks that the creation of new models of computation will be a never-ending story. Unfortunately, the elegant Turing machines do not constitute the final and complete model of computation (despite that many computer scientists and the author of this paper would like to be so).

As it is written in \cite{wegner12} - Turing machines can be considered as an attempt to create the {\em theory of everything} for computer science, whereas similar attempts of complete theories for physics (by Newton, Laplace, Einstein, Hawking), mathematics (by Hilbert, G\"{o}del, Church, Turing) or philosophy (by Aristotle, Plato, Hegel) have failed.  If Turing machines were truly complete, computer science with its Turing machine model would be an exception from other sciences, and computer science together with its Turing machine model would be complete. If so, by reduction techniques, we could prove also completeness of mathematics (decision problem in mathematics \cite{hilbert28} - disproved by G\"{o}del, Church and Turing \cite{godel31,church36,turing36}), and completeness of physics, philosophy, medicine, economy and so on.

Problems can be either solvable or unsolvable (called also undecidable) using a specific model/ theory. In computer science, Turing machines form such dominating and most popular model for problem solving.
Some problems are Turing machine solvable and some not. In particular, this paper intends to provide a new approach to deal with Turing machine undecidable problems. The typical belief is that  proving that a specific problem is TM-undecidable stops any attempt to solve that problem and that is the end of the story. On the other hand, we are convinced that this is only the beginning. First of all, we may decide special instances (or perhaps even almost all instances) of the undeciable problem. For example, if we have the probability distribution of input instances, perhaps randomized techniques may help to estimate which inputs are decidable. Secondly, we can approximate the solutions and we may decide the specific instances either in a finite number of steps or asymptotically in the infinity. There are other approaches possible to deal with undecidability too (see, e.g., the infinity, evolution and interaction principles \cite{eber04a}). Our measure of the unsolvability of computer science problems (although strangely not used yet) is very simple and formal - it is reprented by the percentage of the input instances that are undecidable. For example, the problem that is 100\% undecidable, means that all input instances are undecidable and there are no decidable instances. On the other hand, 0\% of undecidability is typical for recursive problems/languages, and recursively enumerable problems /languages may have non-zero undecidability measure.

The new contributions of this paper are the following:
\begin{itemize}
\item
proposing to use the measure of undecidability of problems unsolvable by Turing machines based on the probability distribution of its input data, i.e., to provide the degree of unsolvabilty based on the number of undecidable instances of input data versus decidable ones,
\item
a new shorter proof of the Hilbert's Entscheidungsproblem of the Turing universal TM \cite{turing36} based on Cantor's results, 
\item 
the extension of Turing oracle machines (Turing $o$-machines) \cite{turing39} to the whole class of super-Turing (hypercomputational) models of computation, 
\item
introduction of three new complexity classes \cite{eber23} for Turing machine undecidable problems, i.e., U-complete, D-complete, H-complete complexity classes inspired by Cook/Levin NP-complete class for intractable problems \cite{hopcroft01,kleinberg06},
\item
proving that recursive languages are not equal to recursively enumerable languages for U-complete problems that is obvious compared to the mystery relation between polynomial versus nondeterministic polynomial languages for NP-complete problems,
\item
hypothesizing that we have an infinite hierarchy of complexity problems  for undecidable problems  based on Cantor \cite{cantor74} and Turing \cite{turing39}.
\end{itemize}

\medskip
This paper is organized as follows. In section 2, we define Turing machines in the form used currently in computer science textbooks. Section 3 covers work of Cantor used in computer science. In section 4, we describe two solutions of the Entscheidungsproblem and going beyond it. In section 5, we discuss other TM-undecidable problems and three new complexity classes for them. Section 6 contains conclusions and final comments.

\section{Turing Machine}

A Turing machine (TM) is considered a formal model of a (digital) computer
running a particular program. A Universal Turing machine represents all possible TMs.  The Turing Machine is the invention of
Alan Turing, who introduced his $a$-machine in his 1936 paper \cite{turing36}               
as a byproduct to prove
the unsolvability of Hilbert's {\em Entscheidungsproblem}\footnote{done as the proof of unsolvability
of the halting problem of a universal TM}.

TM supposed to model any  human (called a {\em computor})
solving algorithmically an arbitrary problem using mechanical methods.
TM has a finite control (having a finite number of states, called by Turing {\em m-configurations}), and an infinite bi-directional one-dimensional tape of
cells (called by Turing {\em squares}), each cell keeping one symbol, and a read/write head. 
Initially, the tape contains a finite-length string $X_1X_2... X_n$ of symbols from the input
alphabet. All other cells, extending infinitely to the left and right are blank, denoted by $B$. A TM being in some
state reads a symbol from the tape alphabet, 
and moves head one position left or right. These are called the {\em moves} by Turing. A read-write tape serves both as input,
output and unbounded storage device. An abstract tape head is the marker - it marks 
the ``current'' cell, which is the only cell that can influence the move of a TM.

\vspace{1cm}
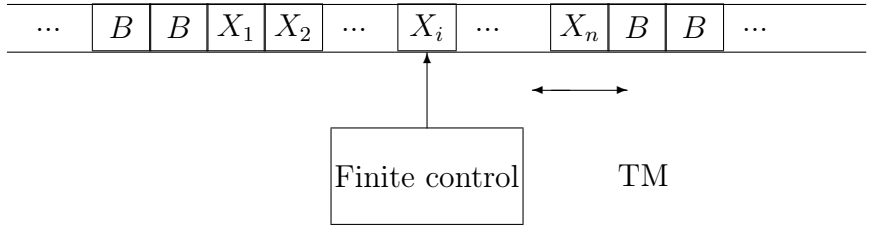
\begin{figure}
\begin{center}
\setlength{\unitlength}{0.01in} 
\begin{picture}(362,161)(84,545)
\put(30,680){\line(1,0){450}}
\put(30,705){\line(1,0){450}}
\put(200,590){\framebox(100,50){Finite control}}
\put(250,640){\vector( 0, 1){40}}
\put(315,660){\vector( 1,0){40}}
\put(325,660){\vector( -1,0){20}}

\put(45,691){\makebox(0,0)[lb]{\raisebox{0pt}[0pt][0pt]{$...$}}}
\put(75,681){\framebox(30,24){$B$}}
\put(105,681){\framebox(30,24){$B$}}
\put(135,681){\framebox(30,24){$X_1$}}
\put(165,681){\framebox(30,24){$X_2$}}
\put(205,691){\makebox(0,0)[lb]{\raisebox{0pt}[0pt][0pt]{$...$}}}
\put(235,681){\framebox(30,24){$X_i$}}
\put(275,691){\makebox(0,0)[lb]{\raisebox{0pt}[0pt][0pt]{$...$}}}
\put(315,681){\framebox(30,24){$X_n$}}
\put(345,681){\framebox(30,24){$B$}}
\put(375,681){\framebox(30,24){$B$}}
\put(415,691){\makebox(0,0)[lb]{\raisebox{0pt}[0pt][0pt]{$...$}}}

\put(350,610){\makebox(0,0)[lb]{\raisebox{0pt}[0pt][0pt]{TM}}}
\end{picture}
\end{center}
\caption{Turing automatic machine}
\end{figure}

\bigskip
\begin{Definition}(Turing machine)
\newline
Formally, a Turing machine (TM) is the 7-tuple and has been defined by several scientists capturing a descriptive ideas from \cite{turing36} (our definition is based on  \cite{hopcroft01}):

\begin{center}
$ M = ( Q, \Sigma, \Gamma, \delta, q_0, B, F)$ , where
\end{center}
\begin{description}
\item{$Q$}
 is a finite set of {\em states}, 
\item{$\Gamma$}
is a finite {\em tape alphabet}, $\Sigma \subseteq \Gamma$ is an 
{\em input alphabet},
\item{$\delta:Q \times \Gamma \rightarrow Q \times \Gamma \times \{L,R\}$}
is a {\em transition (next move) function} (can be undefined for some arguments);
it takes a state and tape symbol, returns a new state and replacement symbol and 
direction $L/R$ (left/right) for head motion; initially, the tape head is at the
leftmost cell that holds the finite input,
\item{$q_0 \in Q$}
is the {\em initial state}; $F \subseteq Q$ {\em accepting (final) states (can be empty)},
\item{$B \in \Gamma - \Sigma$} is a {\em blank symbol} (it appears initially in all
but the finite number of initial cells that hold input symbols).
\end{description}
\end{Definition}

We use the simplicity of the TM model to prove formally that there are
specific problems (languages) that the TM cannot solve. Solving the problem
is equivalent to decide whether a string belongs to the language.

TM language is accepted in two equivalent ways:
\begin{description}
\item{\em by final state}: 
the set of input strings that cause it to rich a final state\\
$L(M)= \{ w \;|\; q_0w \vdash^* \alpha p \beta$ for some $p \in F$ 
and $\alpha,\beta \in \Gamma^*$, $w \in \Sigma^*$, $q_0$ is initial state $\}$.
\item{\em by halting}:  
(used originally by Alan Turing) the set of input strings that cause TM to halt, i.e., to have no next move defined\\
$H(M)= \{ w \;|\; q_0w \vdash^* \alpha p X \beta$ and $\delta(p,X)$ is not defined,
$\alpha,\beta \in \Gamma^*$, $X \in \Gamma$, $w \in \Sigma^*$, $q_0$ is initial state $\}$.
\end{description}

The number of different languages (problems) over any alphabet of more than one
symbol is not countable (justification by diagonalization argument), 
however the number of all possible TMs is enumerable. Thus it is clear
that there exist problems which cannot be mapped (solved) by TMs.
In such a way Turing received his result about unsolvability of the ``halting problem''
of the Universal Turing Machine representing all possible TMs. This corresponds
directly to the unsolvability of Hilbert's {\em Entscheidungsproblem} - the main
reason of introduction by Turing his $a$-machines.

We will speak that one model is {\em more expressive} than another if it allows to decide
about a larger set of strings (languages/problems).

\bigskip
\begin{Remark}
There are several extensions of TMs that turned out to be equivalent to TMs. These are multitatpe TMs, nondeterministic TMs, semi-infinite tape TMs, TMs with left L, right R, and no move N, multistack machines  and counter machines. 
There are also mutliple models of computation that turned out to be equivalent to TMs, e.g., Markov normal algorithms, Post systems, G\"{o}del functions, Knuth attribute grammars, van Wijngaarden two-level grammars, Chomsky-0 grammars, RAM machines, Pidgin Algol, Herbrand functions, Kleene recursive functions, Church $\lambda$-calculus. 
\end{Remark}

\bigskip
\begin{Remark}
There are also models of computations less expressive than TMs, e.g., finite automata, pushdown automata, linearly bounded automata, Chomsky-3, Chomsky-2, Chomsky-1 grammars.
\end{Remark}

\bigskip
\begin{Remark}
It turned out that there are also models of computation, more expressive than TMs. They are called {\em super-Turing} (or {\em hypercomputational}) models of computation. Three of them were proposed by Turing himself, i.e., choice $c$-machines {\em \cite{turing36}}, oracle $o$-machines {\em \cite{turing39}}, and unorganized $u$-machines {\em \cite{turing48}}. Other examples include cellular automata, neural networks, Interaction machines, Persistent TMs, Site and Internet machines, $\pi$-calculus, \$-calculus, Inductive TMs, Infinite time TMs, Accelerating TMs, Evolutionary TMs and Evolutionary automata. They derive higher expreesiveness using three principles {\em \cite{eber04a}}: {\em interaction} (by interacting with the external component), {\em evolution} (by allowing adaptability of TM), {\em infinity} (initial configuration, number of components, time, memory, alphabet).
\end{Remark}

\section{Cantor and Computer Science}

Georg Cantor is considered as one of the founders of the set theory and foundations of mathematics \cite{whitehead10}. His results were heavily used by Alan Turing in \cite{turing36, turing39}, thus, indirectly, he influenced computer science too.

The main contributions of Cantor that are crucial for computer science are the following:

\begin{itemize}
\item
Cantor introduced the one-to-one correspondence, allowing to prove that there are many types of infinite sets;
\item
Cantor proved that the set of real numbers is larger than the set of natural numbers (uncountable and countable infinities) \cite{cantor74} (used by Turing in \cite{turing36});
\item
Cantor proved that the power set of any set is strictly larger than the set itself;
\item
Cantor introduced the diagonal argument \cite{cantor91} (used by Turing in \cite{turing36});
\item
Cantor introduced continuum hypothesis that there exists no set whose power is greater than that of the naturals and less than that of the reals;
\item
Cantor introduced infinite sequences of cardinalities and ordinals and their arithmetic (used by Turing in \cite{turing39}).
\end{itemize}

\section{Two Solutions and beyond the Entscheidungsproblem}

The Entscheidungsproblem is the first undecidable problem from the uncountable number of undecidable problems. Although very important, it is an instance of one unsolvable problem only. To concentrate exclusively on it (or its equivalent: the Halting problem of Universal TM), would be equivalent to restriction of the whole class of NP-complete problems to one SAT ({\em Satisfiability}) problem only.

\subsection{Solution of the Entscheidungsproblem by Alan Turing}

Turing, in a smart way, selected real numbers only (we can think about this as the encoding by decimals or binaries of any logical statement in mathematics which always is possible), to disprove the solvability of the Entscheidungsproblem. He proved that his $a$-machines cannot compute algorithmically by finite means all definable real numbers (having by Cantor an uncountable cardinality), but only its computable subset (defined by countable number of $a$-machines). 

\bigskip
\begin{Remark}
In {\em \cite{aberth68}} a more precise clarification of computable numbers has been provided: ``A computable number may be defined roughly as a (real) number for which rational approximations with an arbitrarily small error may be obtained by final means. Thus the familar numbers $1/2$, $e$, $\pi$, $\sqrt 2$ are all computable, while at the same time not all real numbers are, since the computable numbers form a countable set" (and we have an uncountable number of real numbers, thus we do not have enough recursive algorithms to compute all real numbers). The author presented also examples of un-computable numbers, based of nonconvergent sequences violating the above condition. Probably, the simplest example of un-computable numbers, would be finite random sequences of 0's and 1's generated by a uniform random number generator. It is obvious that in most cases they would generate un-computable sequences of binary representation of real numbers, because we have more un-computable real numbers than computable ones.
\end{Remark}

\bigskip
In sections 1-6 of \cite{turing36} Turing defined auxiliary notions, and provided examples of specific $a$-machines. His $a$-machines accepted by halting, i.e., they stopped when the next move was undefined. Turing called his $a$-machine {\em circular} if it printed a finite number of $0$'s and $1$'s on its tape, and {\em circle-free} otherwise.

In section 7, Turing defined finally a universal $a$-machine, i.e., a machine able to encode any $a$-machine (its input and moves) and to emulate motions and work of any such $a$-machine on a universal $a$-machine). In such a way, he represented all logical statements of Entscheidungsproblem by different instances of $a$-machines, and their decidability to halting of his universal $a$-machine. A universal $a$-machine was useful to generate a sequence of computable numbers printed by $a$-machines used in Turing's proof of uncountability.

In section 8, Turing used Cantor's diagonal argument, to prove that his universal $a$-machine, printing  a finite sequence of computable numbers represnting all $a$-machines, will not be able to print the number from the negation of diagonal which will be a defined real number, but not computable (there will no $a$-machine representing it). 

As Turing wrote (\cite{turing36}, pp.246) in his diagonal argument: ``If the computable sequences are enumerable, let $\alpha_n$ be the $n$-th computable sequence, and let $\phi(m)$ be the $m$-th figure in $\alpha_n$. Let $\beta$ be the sequence with $1 - \phi(n)$ as its $n$-th figure. Since $\beta$ is computable, there exists a number $K$ such that $1 - \phi_n(n) = \phi_K(n)$, all $n$. Putting $n = K$, we have $1 = 2 \phi_K(K)$, i.e. $1$ is even. This is impossible. The computable sequences are therefors not enumerable''. The fallacy in this argument lies in the assumption that $\beta$ is computable (and that there exists a corresponding $a$-machine that halts (is not ``circle-free'' using Turing's terminology)).

Finally, in section 11, Turing  justified that the results from section 8, can be used to prove that the Hilbert Entscheidungsproblem can have no solution. This is done by showing that there will be no $a$-machine able for given logic formula to print/decide either $0$ (false) or $1$ (true). It is interesting that Alan Turing despite using two explicit Lemmas in proof of the unsolvability of the Entscheidungsproblem, never labels it as the explicit theorem (perhaps, he considered this as obvious or he was too modest). We will correct that in below explicit theorem using the above Turing's proof 
from \cite{turing36}.

\bigskip
\begin{Theorem}(Turing 1936 \cite{turing36} - On undecidability of Hilbert Entscheidungsproblem)
\newline
There can be no general process for determining whether a given formula $\cal U$ of the functional calculus {\bf {\sf K}} is provable, i.e. that there can be no $a$-machine which, supplied with any one  $\cal U$ of these formulae, will eventually say whether $\cal U$ is provable. $\Box$
\end{Theorem}

\subsection{Solution of the Entscheidungsproblem based on Georg Cantor's work} 

There are now in literature muliple solutions of Entscheidungsproblem and the halting problem of UTM. Of course, there are many solutions using self-reference as we did in our proof. The question erises, can we do it in a simpler way using some results from Cantor? 

We will present an affirmative answer to this problem.

In the proof, the names TM and UTM will be used  instead of $a$-machine and universal $a$-machine. 

\bigskip
\begin{Theorem}(On undecidability of Entscheidungsproblem using Cantor and self-reference)
\newline
The Entscheidungsproblem problem is undecidable.

\noindent
{\em Proof.}
From one point of view UTM is simply TM like all other TMs, and we have an enumerable number of TMs. On the other hand, UTM is a special TM, because it represents all TMs. How will  behave UTM if we put as its input UTM itself, i.e. self-reference? Putting all TMs as UTM input strings, requires to analyze all subsets of strings representing various languages of TMs binary codes. However, by Cantor theorem {\em \cite{cantor74}}, the number of such languages is uncountable, thus we cannot decide that (by halting in this case) in an enumerable (finite) number of steps. This by reduction from halting poblem to Entscheidungsproblem is equivalent to undecidability of the Entscheidungsproblem itself - reduction can be like for example the last step in Turing's proof (section 11 from {\em \cite{turing36}}).
$\Box$
\end{Theorem}

\subsection{Beyond the Entscheidungsproblem: Oracle Machines - Turing's Systems of Logic based on Cantor's Ordinals}

Together with choice $c$-machine \cite{turing36}, Turing oracle $o$-machine \cite{turing39},pp.13 were the first hypercomputational models of computation going beyond Turing $a$-machine \cite{turing36}. In particular, $o$-machines were able to solve the halting problem of UTM and Entscheidungsproblem.

In ancient Greece {\em oracles} were supernatural people able to communicate hidden knowledge from the dieties to ordinary people \cite{eber04a}. Turing did not elaborate neither on $c$-machines nor on $o$-machines \cite{turing36,turing39}. Simply, they were only theoretical vehicles to investigate some mathematical problems, and, for Turing, additionally, to get his Ph.D. under Church's supervision at Princeton \cite{turing39}. In particular, Turing mentions his $o$-machine very briefly in his Ph.D. thesis only, and most of the time he devotes to the study of systems of logic. Even less we know about $c$-machines \cite{turing36}.

$c$-machines supposed to reply to some queries by interaction with the external human operator, when $a$-machine did not know how to decide a theorem: true or false. In other words, such statement was assumed to be not the theorem, but an axiom decided by the human ($c$-machine), and not by $a$-machine.

On the other hand, the $o$-machine was the $a$-machine expanded by the {\em black-box oracle tape} who somehow ``knew'' the answer to the undecidable problems/queries from its host $a$-machine. In particular, it could be the halting problem of universal $a$-machine. Note, however, that $o$-machine cannot solve its own halting problem unless we will add another level oracle tape, and we will have then $o$-machine of the 2-nd order, and so on, up to an infinite number of oracles. This, in the spirit, resembles an infinite number of infinities/ordinals introduced by Georg Cantor,
and what Turing wrote at the beginning of his thesis \cite{turing39} about an infite sequence of more complete logics (we can think about them as more complete hierarchy of $o$-machines, similar like Entscheidungsproblem was investigated using $a$-machines):

\begin{quote}
The well known theorem of G\"{o}del shows that every system of logic is in a
certain sense incomplete, but at the same time it indicates means whereby from
a system $L$ of logic a more complete system $L'$  may be obtained. By repeating
the process we get a sequence $L, L_1 = L', L_2 = L_1 ', L_3 = L_2 ' ,. . .$ of logics
each more complete than the preceding. A logic $L_\omega$ may then be constructed in
which the provable theorems are the totality of theorems provable with the help
of the logics $L, L_1, L_2,. . .$ We may then form $L_{2\omega}$ related to $L_\omega$ the same may
as $L_\omega$ was related to $L$. Proceeding in this may we can associate a system of
logic with any given constructive ordinal. It may be asked whether a sequence
of logics of this kind is complete in the sense that to any problem $A$ there
corresponds an ordinal 
 such that $A$ is solvable by means of the logic $L_\alpha$. I
propose to investigate this problem in a rather more general case, and to give
some other examples of ways in which systems of logic may be associated with
constructive ordinals.
\end{quote}
     
\section{Other Undecidable Problems and Complexity Classes}

This section has been based on \cite{eber23}. Turing machines \cite{turing36} and algorithms are two fundamental concepts of computer science
and problem solving. Turing machines describe the limits of problem solving using conventional recursive algorithms,
and laid the foundation of current computer science in the 1960s.
TMs made imprecise definitions of algorithms mathematically more precise and formal and they led to Kleene's Turing Thesis that every algorithm can be represented in the form of Turing machine.

Note that there are several other models of algorithms, called super-recursive algorithms, that can compute more than Turing machines, using hypercomputational/superTuring models of computation \cite{burgin05,syropoulos08,eber04a}.

It turns out that (TM) {\em undecidable problems} cannot
be solved by TMs and {\em intractable problems} are solvable,
but require too many resources (e.g., steps or memory). For undecidable problems effective
recipes do not exist -
they are covered by several classes of nonrecursive algorithms
(two outer rings from Figure 2). 
 On the other hand, for intractable problems algorithms exist, but running them on a deterministic Turing Machine,
requires an exponential amount of time (the number of elementary moves of the TM) as a function of the TM input.

Let's denote the class of decidable recursive algorithms by 
$REC$, and its complement by $nonREC$.

\begin{figure}[!h]
\vspace*{-2mm} 
\begin{center}
\includegraphics[width=11.6cm]{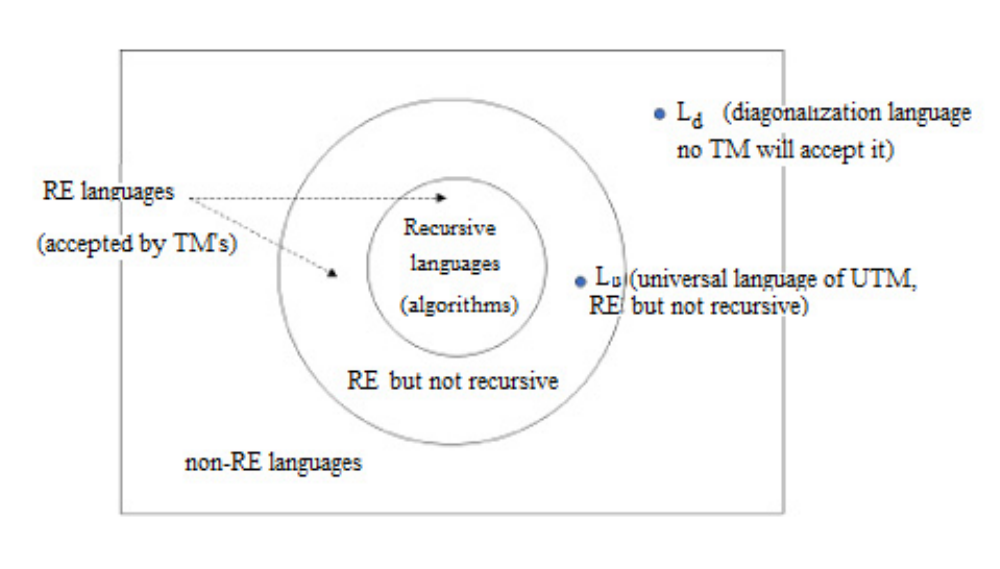}
\end{center}\vspace*{-11mm}
\caption{Relation between Recursive, RE and non-RE Languages}
\end{figure}

\bigskip
We use the simplicity of the TM model to prove formally that there are specific
problems (languages) that the TM cannot solve \cite{hopcroft01}.
Solving the problem is equivalent to decide whether a string belongs to the language.
A problem that cannot be solved by computer (Turing machine) is called
{\em undecidable} (TM-undecidable).
The class of languages accepted by Turing machines are called
{\em recursively enumerable (RE) languages}.
For RE-languages, TM can accept the strings in the language but cannot tell
for certain that a string is not in the language.

\bigskip
There are two classes of Turing machine unsolvable languages (problems):
\begin{description}
\item{\em recursively enumerable but not recursive (REnonREC)}  - TM  can accept the strings in the language
but cannot tell for certain that a string is not in the language (e.g., the language of the universal
Turing machine, or Post's Correspondence Problem languages).
A language is decidable but its complement is undecidable.
We can denote that class as $RE \cap nonREC$, but $REnonREC$ is shorter.
\item{\em non-recursively enumerable (nonRE)} - no TM can even recognize the members of the language
in the RE sense (e.g., the diagonalization language). Neither a language nor
its complement is decidable, or a language is undecidable
but its complement is decidable. 
\end{description}

Decidable problems have a (recursive) algorithm, i.e., TM halts whether or not it accepts its input.
Decidable problems are described by {\em recursive languages}.
Recursive algorithms as we know are associated with the class of recursive languages,
a subset of recursively enumerable languages for which we can construct its accepting TM.
For recursive languages, both a language and its complement are decidable.

Turing machines are used as a formal model of classical (recursive)  algorithms.
An algorithm should consist of a finite number of steps, each having well defined
and implementable meaning.
We are convinced that computer computations are not restricted to such restrictive
definition of algorithms only. If we allow for an infinite number of steps (e.g., reactive programs)
and/or not well defined/implementable meaning
of each step (e.g., an Oracle), we are in the class of super-recursive algorithms \cite{burgin05}.

Such new types of computation and computational models are called very often hypercomputation (or super-Turing computation) and hypercomputational (or super-Turing) models of computation.

\bigskip
\begin{Definition}(On super-Turing computation)
\newline
By {\em super-Turing computation} (also called {\em hypercomputation}) we mean any computation that cannot be carried out
by a Turing machine as well as any (algorithmic) computation carried out by a Turing
machine.
\end{Definition}

\bigskip
\begin{Definition}(On time complexity)
\newline
We can define four classes of problems/
languages for abstract automata/ machines over discrete finite alphabet $\Sigma$ of at least one symbol to decide strings in the language in the order of increasing hardness and complexity. We will define time complexity counting the number of steps, however an analogous definition for space complexity counting the number of memory cells can be defined:

\begin{enumerate}
{\em
\item
{\bf p-decidable problems:} The number of steps is polynomial in the problem size and problems will be called 
{\em polynomially decidable}\\
({\em p-decidable}).
\item
{\bf e-decidable problems:} The number of steps is at least exponential  in the problem size and problems will be called {\em exponentially decidable}
({\em e-decidable}).
\item
{\bf a-decidable problems:} The number of steps is infinite but computed in finite time, i.e., {\em asymptotically/limit
decidable} ({\em a-decidable})
(analogy: convergent infinite series,
mathematical induction, computing infinite sum in definite integral).
\item
{\bf i-decidable problems:} The number of steps is infinite and requires infinite time
to decide strings in the problem size, i.e., {\em infinitely decidable} ({\em i-decidable}) (undecidable in the finite sense).
}
\end{enumerate}
\end{Definition}

The classical complexity theory usually covers classes (1) as easy/tractable and class (2) as intractable problems.
In class (1) polynomials have usually small constant coefficients and exponents. In class (2) functions can be transcomputationally complex, e.g., $1.5^n$, $2^n$, $n!$, $2^{2^n}$, $2^{2^{2^n}}$. Separation of class (1) and (2) is not rigid. There can be complexities between polynomials and exponentials, e.g., $n^{log_2 n}$.

Class (3) is
an intermediate class because although technically it requires an infinite number of steps
(or memory cells for space complexity), we can find
a solution in the limit using finite resources. These can be for instance infinite geometric series where sum is convergent in infinity and computed by formulas in finite number of steps. The same is with infinite series used to compute numbers like $e$ or $\pi$. The mathematical induction proofs also have an infinite number of steps, however an induction step allows to fold them to finite computations. Class (3) is represented by convergent in infinity subset of
Inductive Turing Machines,
anytime algorithms,
evolutionary algorithms, or \$-calculus.
Class (4) requires infinite resources and is unsolvable by Turing Machine
(but solvable by hypercomputers using infinite resources). Classes (2) , (3) and (4) cover non-polynomial algorithms,
classes (3) and (4)
belong to super-recursive algorithms \cite{burgin05}. Class (1) and (2) cover recursive algorithms.

Obviously, undecidable problems are characterized by computations growing faster than exponentially (i.e., hyperexponentially) or they may have even an infinite computational complexity (e.g., an enumerable or real numbers infinity type).

Note that we have in reality an infinite hierarchy of infinities (cardinalities, ordinals) in mathematics/set theory \cite{kuratowski77}, whereas computer science considers only typically enumerable infinity (denoted by  ordinal numbers $\omega$, $\alpha$   or  cardinality $\aleph_0$) with some extension to true real numbers (denoted by ordinal number $c$ or cardinal number $\aleph_1$) represented by analog computers, neural networks operating on real numbers or evolution strategies operating on vectors of real numbers.

Traditional complexity theory deals with various type of decidable algorithms (see e.g., \cite{kleinberg06}), i.e., graph, greedy, divide and conquer, dynamic programming, network flow, NP-complete, PSpace, approximation, local search, and randomized algorithms. All of them have polynomial or exponential complexities. Algorithms than run forever \cite{kleinberg06} and super-recursive algorithms \cite{burgin05} require an extension of complexity theory to infinite cases (nobody did it so far for enumerable $\aleph_0$  or non-enumerable infinities $\aleph_1$, $\aleph_2$, $\aleph_3$,...). For example, local search, also known as hill climbing, typically allows to find local optima only. However, evolutionary algorithms if the search is complete and uses elitist selection allows in infinity to reach global optimum. Then, we do not solve a given problem in polynomial time nor in exponential time, but, perhaps, in infinite time (classes a-decidable and i-decidable of algorithms). But if infinite, what type of infinity we are talking about, $\aleph_0$ or $\aleph_1$?  Or something else?

It is necessary to make a distinction between recursively undecidable problems and super-recursively undecidable problems. At this moment, we can hypothesize that H-complete class covers super-recursively undecidable problems. Whether all of them? Probably not. This is an open research problem. Our belief is that assuming that Figure 2 covers all possible problems then Turing machine  recursively undecidable problems belong to two outer rings, and super-recursively undecidable problems belong only to the outer ring.

Note also that the granularity of outer ring in Figure 2 is not sufficient, because currently contains both D-complete, H-complete and complement of U-complete languages. In the future, the outer ring has to be partitioned further (probably, in the style of Cantor, leading, most likely, to an infinite hierarchy of undecidable problems).

\begin{Corollary}(On cardinality of unsolvable problems)\\
We have more unsolvable problems than solvable ones.

{\em Proof:} The justification is obvious. It is interesting that we have more undecidable/unsolvable problems (represented by the cardinality of real numbers - an uncountable infinity) than decidable ones (represented by the cardinality of natural numbers - an enumerable infinity). This is caused by the fact that decidable problems are modeled by Turing machines and we have only an infinte enumerable number of Turing machines possible (see, e.g., \cite{hopcroft01}). $\Box$
\end{Corollary}

The new complexity classes will be defined in the order of growing undecidability. Note that we concentrated on time computational complexities inspired by NP-complete problems class. An analogous classification can be provided based on memory computational complexities, e.g., inspired by PSPACE-complete problems class. 

\medskip
Before we will define 3 new complexity classes for Turing Machine undecidable problems, we will present a few examples of typical unsolvable problems.

The Universal TM simulates the work of arbitrary TM. Its language is RE but not recursive.
\bigskip
\begin{Definition}(On the Universal TM Language)
\newline
The {\em universal TM language} $L_u$ (of the Universal Turing Machine) UTM $U$ consists of the set
of pairs $(M,w)$, where $M$ is a binary encoding of TM and $w$ is its binary input. The UTM $U$ accepts $(M,w)$ iff
TM $M$ accepts $w$ {\em \cite{hopcroft01}}.
\end{Definition}

The diagonalization language $L_d$ is an example of the language that is believed
even more difficult in solvability than $L_u$, i.e., language of UTM accepting words $w$ for arbitrary
TM $M$. $L_d$ is non-RE, i.e., it does not exist any TM accepting it.

\bigskip
\begin{Definition}(On the Diagonalization Language)
\newline
The {\em diagonalization language} $L_d$  consists of all strings $w$ such that TM $M$ whose
code is $w$ does not accept when given $w$ as input {\em \cite{hopcroft01}}.
\end{Definition}

The existence of the diagonalization language that cannot be accepted by any TM
is proven by the diagonalization table with
``dummy'', i.e., not real/true values. Of course, there are many diagonalization
language encodings possible that depend how transitions of TMs are encoded.
This means that there are infinitely many different $L_d$ language instances (but
nobody wrote a specific example of $L_d$).
Solving the halting problem of UTM, can be used for a ``constructive''
proof of
the diagonalization language i-decidability demonstrating all strings belonging to the language.

\bigskip
\begin{Definition}(On Nonempty and Empty TM Languages)
\newline
The {\em Nonempty TM language} $L_{ne}$ consisting of all binary encoded TMs whose language is not empty, i.e., $L_{ne} = \{ M\; | \;L(M) \neq \emptyset \}$ is known to be recursively enumerable but not recursive, and its complement - the {\em Empty TM language} $L_e = \{ M\; | \;L(M) =  \emptyset \}$ consisting of all binary encoded TMs whose language is empty is known to be non-recursively enumerable {\em \cite{hopcroft01}}.
\end{Definition}

\bigskip
\begin{Definition}(The Post Correspondence Problem (PCP))
\newline
The TM undecidable {\em Post Correspondence Problem (PCP)} {\em \cite{post46}} asks, given two lists of the same number of strings over the same alphabet, whether we can pick a sequence of corresponding strings from the two lists and form the same string by concatenation.
\end{Definition}

\bigskip
\begin{Definition}(On Busy Beaver Problem (BBP))
\newline
The TM undecidable {\em Busy Beaver Problem (BBP)} {\em \cite{rado62}} considers a deterministic 1-tape
Turing machine with unary alphabet $\{1\}$ and tape alphabet $\{1,B\}$, where $B$
represents  the tape blank symbol. TM starts with an initial empty tape and accepts by
halting. For the arbitrary number of states $n=0,1,2,...$ TM tries to compute  two functions:
the maximum number of 1s written on tape before halting (known as the busy beaver
function $\Sigma(n))$, and the maximum number of steps before halting (known as the
maximum shift function $S(n))$.
\end{Definition}

In \cite{bringsjord12}, the author (inspired by the recent famous or infamous - pending on the point of view) the world economy crisis from 2008) introduced, partially for ``fun'', two other undecidable problems, related to BBP, namely the Economy Collapse Problem (ECP) and the Economy Immortality Problem (EIP).

\bigskip
\begin{Definition}(The Economy Collapse Problem (ECP))
\newline
Let $ECP(n)$ be the maximum amount of time for which any economy with $n$ states can function without collapsing {\em\cite{bringsjord12}}.  Here collapse can be equated with the corresponding Turing machine $M$
halting when started on a blank tape.  Compute $ECP(n)$ for
arbitrary values of $n$.
\end{Definition}

\bigskip
\begin{Definition}(The Economy Immortality Problem (EIP))
\newline
Let $EIP$ represent economies that never collapse, i.e., are
immortal {\em \cite{bringsjord12}}.  For arbitrary values of $n$ decide whether any economy with $n$ states is immortal.
\end{Definition}

\subsection{Undecidability complexity classes: U-complete, D-complete and H-complete problems}

Now we are ready to introduce 3 new classes of TM undecidable problems, inspired by the NP-complete class definition. Surprisingly, never such new complexity classes have been defined explicitly for undecidable problems. This would be equivalent that the NP-complete class for intractable problems was left undefined. It sounds strange indeed. We will try to correct that situation, because reduction techniques are commonly used for undecidable problems \cite{hopcroft01} without explicit definition of various undecidable classes of problems.

\bigskip
\begin{Definition}(On U-complete languages)
\newline
We say a language $L$ is {\em U-complete (Universal Turing Machine complete)} iff
\begin{enumerate}
\item
Any word $w$ can be decided in a finite number of steps if  $w \in L$, or it requires an infinite number of steps if  $w \notin  L$ (semi-decidability condition).
\item
For any language $L'$ satisfying (1) there is p-decidable, or e-decidable reduction of $L'$ to $L$ (completeness condition).
\end{enumerate}
\end{Definition}

\bigskip
\begin{Corollary}
U-complete languages belong to the {\em REnonREC} class.
$\Box$
\end{Corollary}

Examples of U-complete languages include $L_u$ (a basic representative to call the whole class), PCP, $L_{ne}$, BBP, ECP, ambiguity of context-free grammars problem, nontrivial properties of RE languages, global optimization problem.

To prove that any new problem belongs to U-complete class, we check firstly that it is semi-decidable, i.e., words from the language are accepted and words outside of the language  cannot be decided. The 2nd part of the proof is done by reduction techniques using any known existing member from U-complete class. For details you can see for instance \cite{hopcroft01}, thus we skipped appropriate proofs. Similarly is with new instances of D-complete and H-complete classes. Various reduction examples has been provided in the literature, thus we will not repeat that effort.

The {\em U-hard languages}, a superset of U-complete languages, satisfy only completeness condition from the above definition.

\begin{Corollary}(On analogy of U-complete semi-decidable class with NP-complete class) \\
For U-complete semi-decidable class if you can guess words in the language you can decide acceptability of the word in finite number of steps.

{\em Proof:} Note that U-complete semi-decidable languages resemble NP-complete class - if you know the solution (the string in the language) you can decide in a finite number of steps that the string is accepted, similar like in NP class you can decide in a polynomial time about string acceptability. For this reason it starts a hybrid  ``easiest class'' in the hierarchy of TM undecibale problems, followed by D-complete and H-complete languages.$\Box$
\end{Corollary}
However, the similarity with NP-complete class ends here, because U-complete languages are undecidable. For NP-complete class $P\: ?=\: NP$ remains one of the unsolved mysteries of computer science, whereas the similar answer for $REC\: ?=\: RE$ is straightforward.

\begin{Theorem}(On undecidability of U-complete class)\\
$REC \neq RE$.

{\em Proof:} is obvious. U-complete languages cover the ring $REnonREC$ outside of the $REC$ decidable class. Language of the Universal TM $L_u$ has been proven to be undecidable. All other U-complete languages can be obtained by reduction from $L_u$ by definition, thus they are undecidable too.  
$\Box$
\end{Theorem}

\bigskip
\begin{Definition}(On D-complete languages)
\newline
We say a language L is {\em D-complete (Diagonalization complete)} iff
\begin{enumerate}
\itemsep=0.95pt
\item
Any word $w$ from $L$ cannot be decided in a finite number of steps (undecidability condition).
\item
For any language $L'$ satisfying (1) there is p-decidable, or e-decidable reduction of $L'$ to $L$ (completeness condition).
\end{enumerate}
\end{Definition}

\bigskip
\begin{Corollary}
D-complete languages belong to the {\em non-RE} class.$\Box$
\end{Corollary}

Examples of D-complete languages include $L_d$ (a basic representative to call the whole class), $L_e$, EIP, complement of $L_d$, complement of $L_u$, complement of BBP, complement of ambiguity of CFG.

The {\em D-hard languages}, a superset of D-complete languages, satisfy only completeness condition from the above definition.

\bigskip
\begin{Definition}(The hyper-diagonalization language)
\newline
The {\em hyper-diagonalization language} $L_{hd}$ consists of all strings $w$ such that TM $M$ whose code is $w$ will not  accept even in an infinite number of steps when given $w$ as input.
\end{Definition}

\bigskip
\begin{Definition}(On H-complete languages)
\newline
We say a language $L$ is {\em H-complete (Hypercomputation complete)} iff
\begin{enumerate}
\itemsep=0.95pt
\item
Any word $w$ from or outside of $L$ cannot be decided in an infinite number of steps (hyper-undecidability condition).
\item
For any language $L'$ satisfying (1) there is an a-decidable or i-decidable reduction of $L'$ to $L$ (completeness condition).
\end{enumerate}
\end{Definition}

\bigskip
\begin{Corollary}
H-complete languages belong to the {\em non-RE} class.$\Box$
\end{Corollary}

A canonical representative of this class is $L_{hd}$ .

The {\em H-hard languages}, a superset of H-complete languages, satisfy only completeness condition from the above definition.

We can conclude the following about undecidable problems classes.

\begin{Hypothesis}(On infinite hierarchy of undecidable complexity classes)\\
We have an infinite hierachy of undecidable classes.

{\em Proof Outline:} It is based on section 5 of this paper and Cantor's hierarchy of infinite  cardinalities \cite{cantor74,cantor91} and Turing's infinite hierarchy of more complete logics (see \cite{turing39} and subsection 4.3 from this paper). $\Box$
\end{Hypothesis}

\section{Conclusions and Final Comments}

This paper concentrates on canonical contributions of Alan Turing leading to foundations of computer science and programmable computers. This includes both Turing $a$-machines and superTuring $o$-machines. It was stressed that in $a$-machines as well as in $o$-machines Turing used Cantor's seminal results (e.g., the diagonalization, cardinalities of power sets, and the infinite sequence of ordinals).
The above means that Georg Cantor also indirectly influenced computer science and his contributions to computer science deserve to be remembered. We can say even more: without Cantor there would be no Turing's results. However, this does not diminish the historical achievements of Alan Turing. 

Probably, the most concise and convincing answer  to the paper's title  question who: Alan Turing or Georg Cantor contributed more to computer science would be the following. By Alan Turing problems (languages) are supersets of algorithms (Turing machines) that are encoded as strings of natural numbers. By Georg Cantor the all supersets of natural numbers form real numbers, thus we have more undecidable problems than decidable algorithms. In other words, we have more undecidable problems than decidable ones, and paradoxically, computer science spends most of its efforts on decidable problems, and from them more time and efforts on ``easy'' polynomial algorithms rather than on ``hard'' exponential ones.

In the paper it is proposed to update Turing historical results by adding U-complete, D-complete and H-complete complexity classes for undecidable problems (never done explicitly before). A shorter proof of Entscheidungsproblem \cite{hilbert28,turing36} has been presented using Cantor's ideas too. 
It is interesting that accidentally somehow similar approach as used by Turing \cite{turing39} (and based on Cantor's infinite cardinalities) for completeness of logics was used in \cite{eber23} to obtain completeness of cost metrics and meta-search algorithms in \$-calculus.   

Other contributions of Alan Turing to cryptology (Enigma, Colossus), u-machines (genetic algorithms, reinforcement learning, neural networks) \cite{turing48}, robotics, artificial intelligence, artificial life and computational biology \cite{eber04a} are outside of the scope of this paper. 

Currently, we are exhibiting the time of booming machine learning and artificial intelliegnce. Alan Turing, besides his Turing machines and algotirhms, was the precursor of artificial intelligence too. It looks that computers finally are able to pass the Turing test for intelligence \cite{turing48}. However, the progress of research on deep neural networks and deep language models did not allow to achieve general artificial intelligence so far \cite{eber26}.

Of course, there are much more interesting ideas than those covered in this paper. For sure, the results, e.g., of Alfred Tarski, Stan Ulam, Kurt G\"{o}del would require re-visitation, more elaborations, and, hopefully,  would be done in the future.

\section*{Acknowledgments}
This paper has been written in honor of late Professor Peter Wegner from Brown University and late Professor Mark Burgin from University of California Los Angeles - my co-authors of several research papers, with whom I had the honor and pleasure to cooperate for many years. The author appreciates useful comments from anonymous reviewers and editors.

\end{document}